\documentclass[conference]{IEEEtran}
\usepackage{amsmath,amssymb,amsfonts,amsthm}
\usepackage{algorithmic}
\usepackage{algorithm}
\usepackage{graphicx}
\usepackage{textcomp}
\usepackage{xcolor}
\usepackage{booktabs}
\usepackage{multirow}
\usepackage{hyperref}
\usepackage{subcaption}

\graphicspath{{./}}

\newtheorem{theorem}{Theorem}

\begin{document}

\title{Adaptive Two-Level Allocation of a Conserved Capacity Budget\\Across Locations and Service Classes}

\author{
\IEEEauthorblockN{Simone Mainardi}
\IEEEauthorblockA{\textit{Network Security} \\
\textit{Salesforce, Inc.}\\
Dublin, Ireland \\
smainardi@salesforce.com}
\and
\IEEEauthorblockN{Kaushal Bansal}
\IEEEauthorblockA{\textit{Network Security} \\
\textit{Salesforce, Inc.}\\
San Francisco, CA, USA \\
kbansal@salesforce.com}
\and
\IEEEauthorblockN{Prabhat Singh}
\IEEEauthorblockA{\textit{Network Security} \\
\textit{Salesforce, Inc.}\\
San Jose, CA, USA \\
prabhatsingh@salesforce.com}
}

\maketitle

\begin{abstract}
We study how to share a single conserved capacity budget across many locations and two service classes when demand is uneven, time-varying, and can exceed supply. The shape recurs: an origin's request-rate cap split across its edge locations, a licensed throughput cap across premium and standard tenants, or an egress budget between latency-critical and batch workloads. We present a two-level algorithm. The first level redistributes capacity within a class across locations by proportional deficit and excess redistribution; the second lends capacity elastically between classes when one has surplus and the other deficit. We prove it conserves the budget exactly, preserves non-negativity, and reaches a stable allocation in one iteration under stationary demand because it carries no per-cycle state, at $O(KN)$ cost per cycle for $K$ classes and $N$ locations. We evaluate it defending a content delivery network's per-domain budget under volumetric attack, where the classes are confirmed-legitimate and not-yet-cleared traffic; across 8 contention scenarios on a 22-location topology it serves 66--93\% of high-priority demand, competitive with a single-class linear-programming optimum, while never leaving capacity idle or over-committing whenever aggregate demand meets or exceeds the budget (the contention regime these scenarios evaluate). Two findings carry beyond the application. First, a throughput-maximizing objective is wrong under contention: a two-class LP maximizing total served load serves less high-priority load than our demand-proportional, reservation-respecting allocator in most scenarios, because it cannot tell that some load it serves is the contention. Second, inter-class borrowing earns its complexity under bursty load, improving high-priority service by 1.5 points (isolated by ablation), and is neutral under stationary demand. A 5-location prototype with real HTTP traffic validates the pipeline.
\end{abstract}


\section{Introduction}
\label{sec:intro}

A recurring problem in distributed systems is sharing one fixed capacity budget across many locations and a small number of service classes. A controller holds a conserved budget $C$ and must, every decision cycle, decide how much of it each location receives, subject to the constraint that the allocations sum to $C$. The budget may be the sustainable request rate of an origin server divided across the edge locations that front it~\cite{cdn_survey,taiji2019}, an inter-datacenter bandwidth budget split between services of different priority~\cite{bwe2015,swan2013}, or a licensed throughput cap shared by tenants across data centers. In each case the trade-off is the same: allocate too much to a quiet location and capacity sits idle that a busy location needs; allocate too little to a busy one and its demand is denied. When demand additionally differs in priority, a class that must be protected versus a class that may be served best-effort~\cite{bwe2015}, the controller must split the same budget across classes as well as locations. This is the problem we study: \emph{two-level allocation of a conserved budget across locations and two service classes.}

Two structural facts make this harder than uniform division. First, demand is uneven across locations and shifts over time, a pattern documented for global edge-to-datacenter traffic~\cite{taiji2019} and CDN workloads~\cite{cdn_survey}, so an equal split starves the locations where load actually originates while leaving capacity idle elsewhere. Second, a controller that sizes each location for its own peak must commit, in aggregate, far more than the shared budget: if every one of $N$ locations is provisioned for the full budget, the worst-case committed capacity is $N$ times the budget, and any cycle in which enough locations are simultaneously busy exceeds $C$. Peak-provisioning that leaves the aggregate under-subscribed on average is the same economics that motivated centralized allocation in production wide-area networks~\cite{swan2013,bwe2015}. The controller must therefore reclaim that over-commitment the moment aggregate demand crosses the budget, which is exactly when uniform and static schemes fail.

We propose an algorithm that operates at two levels. At the first level it redistributes capacity \emph{within} a class across locations in proportion to observed demand, moving capacity from locations with excess to locations with deficit. At the second level it allows elastic borrowing \emph{between} the two classes, so capacity left unused by one class flows to the other when that other class has unmet demand, and returns when demand rises again. Both levels maintain a strict conservation invariant: the sum of all allocations equals the budget $C$ at every cycle. The algorithm is stateless, recomputing allocations from current demand each cycle, which yields one-step stabilization under stationary demand by construction: a single application reaches a fixed point of the reallocation step.

The instantiation that motivated this work, and the one we evaluate, is defending a content delivery network's per-domain capacity budget during volumetric attacks~\cite{ddos_cloud_survey,alcoz2022}. There the locations are edge regions and the two classes are \emph{confirmed-legitimate} traffic, which must be protected, and \emph{not-yet-cleared} traffic, which is served at reduced priority rather than dropped because the evidence against it is not conclusive. We use this instantiation for the evaluation and the deployment sections, but the algorithm, its invariants, and its convergence result are stated and proved for the abstract model and do not depend on the contending load being adversarial. We return to the generality, and to what porting the algorithm to another instantiation would require, in Section~\ref{sec:limitations}.

Our contributions are:
\begin{itemize}
    \item A formal two-level allocation algorithm for a conserved budget shared across locations and two service classes, with provable conservation, non-negativity, and one-step stabilization, computable in $O(KN)$ time per cycle for $K$ service classes and $N$ locations (both defined in Section~\ref{sec:algorithm}).
    \item A comparative evaluation against 7 baselines across 8 contention scenarios, with contending load measured in the class it targets. The baselines span the design space: naive splits, the production heuristic, a DRF-style weighted max-min fair allocator, and three LP optima (single-class, reservation-respecting, and unconstrained two-class). The algorithm is competitive with a single-class LP optimum on high-priority load served, and we identify a counterintuitive effect: it serves more high-priority load than a throughput-optimal two-class LP in most scenarios, because maximizing total served load can mean serving the contention itself (Section~\ref{sec:evaluation}).
    \item An ablation isolating the inter-class borrowing mechanism: it improves high-priority service specifically under temporally bursty load and is neutral under stationary demand, locating exactly where the two-level design earns its complexity (Section~\ref{sec:evaluation}).
    \item A prototype deployment on a 5-location testbed with real HTTP traffic validating the counting and aggregation pipeline that feeds the allocator, from per-process counting through cross-location transport (Section~\ref{sec:deployment}).
\end{itemize}

\section{Problem Statement}
\label{sec:problem}

Consider a controller that owns a single capacity budget $C$ and distributes it across $N$ locations every decision cycle. The budget is the maximum aggregate load a shared downstream resource can absorb without degrading. Demand arrives at each location, is observed by the controller once per cycle, and is served only up to the location's current allocation; demand beyond the allocation is denied or deferred.

\paragraph{Two service classes.} Demand at each location is split into two classes that differ in priority. We call them the \emph{high-priority} class, whose demand should be protected, and the \emph{best-effort} class, whose demand is served when capacity allows but yields first under contention. The budget is partitioned between them: each class is given a \emph{reservation}, $R_{\text{hi}}$ for the high-priority class and $R_{\text{be}}$ for the best-effort class, with $R_{\text{hi}} + R_{\text{be}} = C$. A reservation is only the share \emph{initially} assigned to a class; the algorithm may lend capacity across the two classes within a cycle (Section~\ref{sec:algorithm}), so a reservation is a starting point, not a hard cap. We use exactly two classes throughout and do not model a deeper tier hierarchy. Many priority-sharing problems reduce to this two-class form: premium versus standard tenants, latency-critical versus batch jobs, or, in our evaluated instantiation, confirmed-legitimate versus not-yet-cleared traffic. How demand is assigned to a class is an input to the allocator, not a contribution of this paper; in the evaluated instantiation a separate upstream classifier produces the labels (Section~\ref{sec:system}), and demand deemed unserviceable is rejected upstream, consuming no budget and never reaching the allocator.

\paragraph{Notation.} Index $i \in \{1, \ldots, N\}$ ranges over locations and $c \in \{\text{hi}, \text{lo}\}$ over the two classes. We write $a_{c,i}$ for the capacity \emph{allocated} to class $c$ at location $i$ and $d_{c,i}$ for the observed \emph{demand} of class $c$ at location $i$. Aggregates over locations are $A_c = \sum_i a_{c,i}$ and $D_c = \sum_i d_{c,i}$. When the class is clear from context we drop the subscript and write $a_i$, $d_i$. The conservation requirement is $\sum_c \sum_i a_{c,i} = C$ at every cycle.

\subsection{Why Uniform and Static Splits Fail}

Demand across locations is concentrated, not uniform: in the traffic model we use for evaluation (Section~\ref{sec:evaluation}) a minority of locations carry the majority of demand, with a long flat tail across the rest. Under such skew a uniform split, $C/N$ to every location, throttles the busy locations far below their demand while the idle locations sit on capacity they cannot use: the uniform share is $C/N$, but the busiest location can need several times that, so a uniform split denies much of its high-priority demand.

The opposite extreme, sizing every location for its own peak by giving each the full budget, avoids that throttling during quiet periods but over-commits in aggregate: with all $N$ locations provisioned for $C$, the worst-case committed capacity is $N \cdot C$, an $N\times$ over-commitment against a budget of $C$. The committed total is lower when fewer locations are active, scaling with the active count, but once enough locations are simultaneously busy the budget is exceeded and the over-commitment must be reclaimed in the same cycle, before the shared downstream resource is overwhelmed. A correct allocator therefore cannot rely on either fixed extreme; it must track demand.

\subsection{Design Goals}

We want an allocation algorithm that satisfies four properties at each decision cycle:
\begin{enumerate}
    \item \textbf{Conservation:} $\sum_c \sum_i a_{c,i} = C$ at all times.
    \item \textbf{High-priority service:} Maximize served high-priority demand, $\sum_i \min(a_{\text{hi},i}, d_{\text{hi},i})$.
    \item \textbf{Fairness:} Locations with proportionally similar demand receive proportionally similar satisfaction ratios.
    \item \textbf{Efficiency:} Minimize capacity that is allocated but unused.
\end{enumerate}

\section{Algorithm}
\label{sec:algorithm}

The algorithm runs at the controller on a fixed decision cycle. Each cycle it observes per-location, per-class demand, computes new allocations, and distributes them to the locations for enforcement. It carries no state between cycles: allocations are recomputed from the current demand and the fixed reservations.

\subsection{Intra-Class Reallocation}

For a single class with total reserved capacity $R$ distributed across $N$ locations, let $\mathbf{a} = [a_1, \ldots, a_N]$ be the current allocations and $\mathbf{d} = [d_1, \ldots, d_N]$ be the observed demands.

We classify each location as having either excess capacity ($a_i > d_i$) or deficit ($d_i > a_i$), then redistribute according to Algorithm~\ref{alg:intra}.

\begin{algorithm}[t]
\caption{Intra-Class Reallocation}
\label{alg:intra}
\begin{algorithmic}[1]
\REQUIRE Allocations $\mathbf{a}$, demands $\mathbf{d}$
\ENSURE New allocations $\mathbf{a'}$ with $\sum a'_i = \sum a_i$
\STATE $e_i \leftarrow \max(0, a_i - d_i)$ for all $i$
\STATE $f_i \leftarrow \max(0, d_i - a_i)$ for all $i$
\STATE $E \leftarrow \sum_i e_i$; \quad $F \leftarrow \sum_i f_i$
\IF{$E = 0$ or $F = 0$}
    \RETURN $\mathbf{a}$
\ENDIF
\IF{$E \geq F$}
    \STATE \textit{// Enough excess to cover all deficits}
    \FORALL{$i$ with $f_i > 0$}
        \STATE $a'_i \leftarrow d_i$
    \ENDFOR
    \FORALL{$i$ with $e_i > 0$}
        \STATE $a'_i \leftarrow a_i - \frac{e_i}{E} \cdot F$
    \ENDFOR
\ELSE
    \STATE \textit{// Not enough excess, share proportionally}
    \FORALL{$i$ with $e_i > 0$}
        \STATE $a'_i \leftarrow d_i$
    \ENDFOR
    \FORALL{$i$ with $f_i > 0$}
        \STATE $a'_i \leftarrow a_i + \frac{f_i}{F} \cdot E$
    \ENDFOR
\ENDIF
\RETURN $\mathbf{a'}$
\end{algorithmic}
\end{algorithm}

The key insight is proportionality. When excess exceeds deficit, every deficit location gets exactly what it needs, and excess locations give up proportional to how much spare capacity they have. When deficit exceeds excess, all available excess is distributed proportionally to how much each deficit location needs. This ensures fairness without requiring optimization solvers.

\begin{theorem}[Conservation]
\label{thm:conservation}
Algorithm~\ref{alg:intra} preserves $\sum_i a'_i = \sum_i a_i$.
\end{theorem}

\begin{proof}
In Case 1 ($E \geq F$): deficit locations gain a total of $F$. Excess locations lose $\sum_{i: e_i > 0} \frac{e_i}{E} \cdot F = \frac{F}{E} \sum_{i: e_i > 0} e_i = F$. Net change is zero.

In Case 2 ($E < F$): excess locations lose a total of $E$ (they drop to their demand level). Deficit locations gain $\sum_{i: f_i > 0} \frac{f_i}{F} \cdot E = \frac{E}{F} \sum_{i: f_i > 0} f_i = E$. Net change is zero.
\end{proof}

\begin{theorem}[Non-negativity]
If all inputs are non-negative, all outputs are non-negative.
\end{theorem}

\begin{proof}
In Case 1, deficit locations receive $d_i \geq 0$. Excess locations retain $a_i - \frac{e_i}{E} \cdot F \geq a_i - \frac{e_i}{E} \cdot E = a_i - e_i = d_i \geq 0$.

In Case 2, excess locations receive $d_i \geq 0$. Deficit locations receive $a_i + \frac{f_i}{F} \cdot E \geq a_i \geq 0$.
\end{proof}

\subsection{Inter-Class Elastic Allocation}

The second level handles reallocation between classes. Given $K$ classes with reservations $\{R_c\}$ and aggregate demands $\{D_c = \sum_i d_{c,i}\}$, a class has surplus when $R_c > D_c$ and deficit when $D_c > R_c$. We state the algorithm for general $K$; throughout this paper $K = 2$.

\begin{algorithm}[t]
\caption{Inter-Class Elastic Allocation}
\label{alg:inter}
\begin{algorithmic}[1]
\REQUIRE Class reservations $\{R_c\}$, class demands $\{D_c\}$
\ENSURE Class budgets $\{A_c\}$
\STATE $A_c \leftarrow \min(D_c, R_c)$ for all $c$
\STATE $s_c \leftarrow \max(0, R_c - D_c)$ for all $c$
\STATE $f_c \leftarrow \max(0, D_c - R_c)$ for all $c$
\STATE $S \leftarrow \sum_c s_c$; \quad $F \leftarrow \sum_c f_c$
\IF{$S > 0$ and $F > 0$}
    \FORALL{$c$ with $f_c > 0$}
        \STATE $A_c \leftarrow A_c + \min\left(f_c,\; S \cdot \frac{f_c}{F}\right)$
    \ENDFOR
\ENDIF
\RETURN $\{A_c\}$
\end{algorithmic}
\end{algorithm}

When one class has unused capacity, for example during a window in which best-effort demand momentarily drops, that capacity flows to the other class and returns when the lending class's demand rises again. This inter-class mechanism is what lets the two-level algorithm serve more high-priority demand than a per-cycle optimum when demand varies over time (Section~\ref{sec:evaluation}). It does not contradict LP optimality: a per-cycle LP recomputes from current demand and has no notion of holding capacity across a temporal cycle, whereas our algorithm's elastic borrowing spans cycles. The same mechanism would also be driven by any process that shifts demand between the classes over time, such as labels being refined by an upstream classifier; we do not model such dynamics here, so the benefit we report comes purely from temporal variation in load volume.

\subsection{Combined Two-Level Algorithm}

Each decision cycle executes three steps:
\begin{enumerate}
    \item Run Algorithm~\ref{alg:inter} to determine per-class budgets $\{A_c\}$.
    \item For each class $c$, scale base per-location allocations by $A_c / R_c$.
    \item For each class $c$, run Algorithm~\ref{alg:intra} on the scaled allocations.
\end{enumerate}

The time complexity is $O(KN)$ per decision cycle where $K$ is the number of classes (two throughout this paper) and $N$ is the number of locations. In the deployment described in Section~\ref{sec:system}, with 2 classes and 22 locations, each cycle is a fixed number of passes over $KN = 44$ per-location entries, negligible against the 10-second decision cycle.

\subsection{Convergence}

The algorithm is stateless: each cycle recomputes allocations from the current demand and a fixed base reservation, and never feeds the previous cycle's output back as input. What this buys is stability, not a unique closed-form target: we show below that one application of the reallocation step lands on a fixed point of that step, so a second application with the same demand changes nothing. We are careful about what this does and does not claim. It does \emph{not} claim a unique allocation independent of the input: when a class is under-provisioned (aggregate demand exceeds the class budget), the step has a continuum of fixed points, and two equal-sum input allocations under the same demand can settle on different stable allocations. The result matters less as a convergence guarantee in its own right and more as a contrast with heuristics whose allocation is a discontinuous function of the demand vector, such as the breach-count scheme evaluated in Section~\ref{sec:evaluation}, which can fail to settle when demand noise moves locations back and forth across its threshold.

\begin{theorem}[One-Step Stabilization]
Fix a demand vector $\mathbf{d}$ and let $\mathbf{a}$ be any allocation with $\sum_i a_i = T$. Let $\mathbf{a'}$ be the output of Algorithm~\ref{alg:intra} on $(\mathbf{a}, \mathbf{d})$. Then applying Algorithm~\ref{alg:intra} again to $(\mathbf{a'}, \mathbf{d})$ returns $\mathbf{a'}$ unchanged: one application reaches a fixed point of the reallocation step. The fixed point is not unique in general; it depends on $\mathbf{a}$ when the class is deficit-dominant (Case 2 below).
\end{theorem}

\begin{proof}
Let $T = \sum_i a_i$ (preserved by Theorem~\ref{thm:conservation}). Define $E(\mathbf{a}) = \sum_i \max(0, a_i - d_i)$ (total excess) and $F(\mathbf{a}) = \sum_i \max(0, d_i - a_i)$ (total deficit), and note $E - F = T - \sum_i d_i$ is fixed for given $T$ and $\mathbf{d}$.

After one application, every location that had excess ($a_i > d_i$) or deficit ($a_i < d_i$) is moved toward its demand: in Case 1 ($E \geq F$) deficit locations are set exactly to $d_i$; in Case 2 ($E < F$) excess locations are set exactly to $d_i$. So in $\mathbf{a'}$, one side of the excess/deficit split is pinned to demand. Consider the second application on $\mathbf{a'}$:

Case 1 ($E \geq F$): deficit locations of $\mathbf{a'}$ already satisfy $a'_i = d_i$, so they contribute zero deficit; only the former excess locations retain $a'_i \geq d_i$, and their total excess is $E - F = T - \sum_i d_i \geq 0$, matched by zero deficit. With $F(\mathbf{a'}) = 0$, Algorithm~\ref{alg:intra} returns its input unchanged (the $F = 0$ guard). Case 2 ($E < F$): excess locations of $\mathbf{a'}$ satisfy $a'_i = d_i$, so $E(\mathbf{a'}) = 0$, and the $E = 0$ guard returns the input unchanged. Either way $\mathbf{a'}$ is a fixed point.

Non-uniqueness in Case 2: the deficit locations receive $a'_i = a_i + \frac{f_i}{F}E$ with $f_i = d_i - a_i$, which depends on the input $\mathbf{a}$, so two equal-sum inputs under the same $\mathbf{d}$ can yield different (each still stable) outputs. The claim is one-step stabilization, not a unique target.
\end{proof}

This matches the empirical observation in Section~\ref{sec:evaluation}: the allocation settles in 1 iteration across all 8 scenarios. The reason a \emph{single} run of the deployed controller is fully determined despite the non-uniqueness is that the controller recomputes each cycle from a fixed base allocation rather than from the previous cycle's output, so the input to the step is the same every cycle: determinism here comes from a fixed input, not from a unique fixed point. Because the allocation is a continuous, demand-proportional function of that input, it avoids the instability of schemes whose allocation depends discontinuously on how many locations cross a threshold.

\section{Instantiation: CDN Capacity Defense}
\label{sec:system}

We now ground the abstract model in the deployment that motivated it. In the CDN instantiation the locations are edge regions, the high-priority and best-effort classes are confirmed-legitimate and not-yet-cleared traffic, and the conserved budget is a domain's per-domain limit in requests per minute (RPM), representing the maximum load the domain's origin server can handle without degradation. The decision cycle is 10 seconds. The allocator runs inside a hierarchical counting and control system spanning 22 geographic regions, organized in four tiers:

\textbf{Pod Counter (Tier 1):} A Rust library compiled as both a native shared object (for Nginx via LuaJIT FFI) and a WASM module (for Envoy). It performs lock-free atomic per-domain counting on each request and flushes deltas to the cluster on a short fixed interval.

\textbf{Cluster Aggregator (Tier 2):} A Go service that receives pod flushes, maintains a Redis-backed sliding window, and forwards aggregated counts to the region.

\textbf{Region Aggregator (Tier 3):} A Go service maintaining a G-Counter CRDT per domain. It participates in cross-region synchronization periodically. The CRDT guarantees eventual consistency and partition tolerance without requiring coordination.

\textbf{Control Plane (Tier 4):} A central Go service that pulls global state from all region aggregators, detects hot domains, consumes per-request class labels from an upstream classifier (Section~\ref{sec:labels}), runs the allocation algorithm of Section~\ref{sec:algorithm}, and pushes new allocations back to the regions.

Each stage adds a small fixed delay, so the end-to-end latency from a demand change to a new allocation is on the order of the decision cycle; the exact per-tier intervals are deployment-tuning parameters rather than properties of the algorithm.

\subsection{Where the Class Labels Come From}
\label{sec:labels}

How requests are assigned to the high-priority and best-effort classes is an input to the allocator, not a contribution of this paper; any mechanism that produces a per-request class label can be substituted, and the allocation algorithm is unchanged. In our CDN instantiation an upstream classifier assigns each request to confirmed-legitimate (high-priority) or not-yet-cleared (best-effort), and requests that match threat-intelligence feeds outright are dropped upstream and never reach the allocator. The only property the allocator relies on is that labels may be \emph{refined over time}: as an attack progresses and the classifier gains confidence, actors move out of the not-yet-cleared class, reducing best-effort demand and freeing capacity that inter-class borrowing can lend to the high-priority class. We do not model these classifier dynamics in the evaluation; the temporal benefit we report (Section~\ref{sec:evaluation}) comes purely from variation in load volume.

\section{Evaluation}
\label{sec:evaluation}

We evaluate in the CDN instantiation, and this section uses its vocabulary: locations are regions, high-priority demand is legitimate traffic, and the contending load is attack traffic.

\subsection{Experimental Setup}

We evaluate on a simulated CDN with 22 regions, matching our deployment topology. The total capacity budget is 10,000 RPM, partitioned into a 75\% high-priority reservation and a 25\% best-effort reservation. Of the legitimate traffic itself, 85\% is high-priority and 15\% is best-effort; the high-priority reservation is sized slightly above the high-priority share so that peacetime traffic fits within it. Legitimate traffic runs at 80\% of capacity on average (75\% in the diurnal scenario, which adds a daily swing).

Demand is geographically skewed. We assign region weights so that a minority of regions carry most of the demand: in the default configuration the three busiest regions together draw a little over 40\% of demand and the top eight about two-thirds, with the remaining regions sharing a flat tail. Long-tailed geographic concentration, where a small number of regions dominate request volume, is a qualitative pattern reported for CDN and edge traffic~\cite{cdn_survey,taiji2019}; the specific concentration here is a parameter of our traffic model, not a measurement of any one deployment. We model time-of-day demand with a multiplicative diurnal factor $1 + \alpha\sin(2\pi t / 100)$ applied to the baseline, where $t$ is the iteration index and 100 iterations represent one daily cycle. The amplitude $\alpha$ is 0.3 for the seven non-diurnal scenarios, a $\pm 30\%$ swing around the mean; the diurnal scenario uses $\alpha = 0.25$ so its peak sits near capacity. Per-region demand additionally carries multiplicative noise, a $\pm 30\%$ jitter drawn independently per region per cycle and clipped to $[0.7, 1.3]$; the clip keeps demand strictly positive and bounded so a single draw cannot dominate or zero out a region, matching the $\pm 30\%$ diurnal amplitude. The bound is a modeling choice, not a measured quantity, and the results are not sensitive to its exact value.

All experiments use 200 iterations per run with 5 random seeds each. Attack traffic starts at iteration 20. The first 20 iterations let the system settle into a representative pre-attack allocation and demand level, so that the attack phase measures adaptation from a realistic operating point rather than from a cold start. We report metrics averaged over the attack phase (iterations 20-200). Attack load is not confined to one class: each scenario splits it between the two classes by a per-scenario fraction, with the application-layer portion that mimics legitimate users landing in the high-priority class and the purely volumetric portion in the best-effort class. This split is deliberate, it forces genuine contention inside the high-priority class rather than measuring a class the attack never reaches, and the fraction is varied across scenarios (from best-effort-heavy volumetric floods to high-priority-heavy application-layer attacks) so the evaluation covers both.

\subsection{Contention Scenarios}

We evaluate 8 scenarios designed to create real capacity contention (total demand exceeds capacity). They are described in the CDN instantiation's terms (an attack is the source of the excess load), but each is simply a spatial-temporal pattern of demand that exceeds the budget:

\begin{itemize}
    \item \textbf{Concentrated:} 100\% capacity excess on a single region.
    \item \textbf{Distributed:} 80\% capacity excess spread across all 22 regions.
    \item \textbf{Rotating:} Excess shifts between 3 regions every 10 iterations.
    \item \textbf{Pulse:} On/off bursts at 120\% capacity every 5 iterations.
    \item \textbf{Slow ramp:} Gradual increase from 0 to 150\% capacity over 40 iterations.
    \item \textbf{Diurnal + Attack:} Excess load during the peak-traffic hour.
    \item \textbf{Contention:} Simultaneous excess on 4 regions.
    \item \textbf{Asymmetric:} Excess targets quiet regions (opposite of where high-priority demand concentrates).
\end{itemize}

\subsection{Baselines}

\begin{itemize}
    \item \textbf{Static Uniform:} Each region gets $C/N$ regardless of demand.
    \item \textbf{Divide-by-Breach:} Divide the budget equally among the locations currently over their limit; within a location, a fixed 75/25 ratio splits capacity between the classes. This is representative of the simple over-limit-counting heuristics common in production rate limiters prior to demand-aware allocation, included only to show how such a scheme behaves on these scenarios; we make no claim that it is optimal or a state-of-the-art baseline. It has three structural weaknesses the results make precise: it cannot distinguish a location far over its limit from one slightly over (both get the same share); it cannot move capacity left unused by one class to the other; and its allocation is a discontinuous function of the over-limit count, so it can fail to settle when many locations sit near the threshold.
    \item \textbf{Proportional:} Allocate proportional to observed demand.
    \item \textbf{LP Optimal (1-class):} Water-filling solution maximizing $\sum_i \min(a_i, d_i)$ subject to $\sum_i a_i = R_c$ within each class reservation. This is the theoretical optimum for single-class allocation.
    \item \textbf{LP Optimal (2-class):} Water-filling over the pooled budget across both classes and all regions, maximizing $\sum_c \sum_i \min(a_{c,i}, d_{c,i})$ subject to $\sum_{c,i} a_{c,i} = C$. This is the throughput optimum for the full two-dimensional problem, and the correct upper bound for the inter-class-borrowing claim. It is allowed to move capacity across both classes, so the proposed algorithm cannot exceed it on total served traffic within a cycle; comparing against it prevents overstating the contribution.
    \item \textbf{LP Optimal (reservation-respecting):} The point between the two LP baselines above. It first guarantees each class a floor of $\min(D_c, R_c)$ within its own reservation, then pools the leftover budget $C - \sum_c \min(D_c, R_c)$ across all (class, location) cells by water-filling on residual demand. Unlike the 2-class LP it cannot starve the high-priority class to serve contending load, and unlike the 1-class LP it can lend a class's idle capacity to the other. It is the \emph{optimal} version of the borrow-above-floor policy the proposed algorithm implements heuristically, so it is the tightest reference for that policy.
    \item \textbf{Weighted Max-Min (DRF-style):} Weighted max-min fair allocation over all (class, location) cells by weighted water-filling: raise a common level $\lambda$ and give each cell $\min(d_{c,i}, \lambda w_c)$, with $\lambda$ chosen so the allocation sums to $C$. The per-class weight $w_c$ is the class reservation, so the high-priority class is weighted three times the best-effort class, matching the reservation split rather than adding a new knob. This follows the Dominant Resource Fairness principle~\cite{drf2011} that the fairness objective must track the structure of demand, and mirrors the per-class weighted max-min fairness used inside SWAN~\cite{swan2013} and BwE~\cite{bwe2015}. It is the principled fairness baseline the rest of the set lacks.
\end{itemize}

\subsection{Metrics}

\begin{itemize}
    \item \textbf{High-Priority Served (\%):} Fraction of high-priority demand that can be satisfied, $\frac{\sum_i \min(a_{\text{hi},i}, d_{\text{hi},i})}{\sum_i d_{\text{hi},i}}$. In the CDN instantiation this is legitimate traffic served, and we use the two terms interchangeably below.
    \item \textbf{Capacity Utilization (\%):} Fraction of allocated capacity that is actually used, $\frac{\sum_i \min(a_i, d_i)}{\sum_i a_i}$.
    \item \textbf{Demand-Weighted Fairness:} Jain's fairness index applied to per-region satisfaction ratios.
    \item \textbf{Convergence:} Number of iterations until allocation changes drop below 5\% for 3 consecutive iterations.
\end{itemize}

\subsection{Results}

Table~\ref{tab:results} summarizes performance across all 8 scenarios.

\begin{table}[t]
\caption{Performance across 8 contention scenarios (mean / worst-case)}
\label{tab:results}
\centering
\small
\begin{tabular}{lcccc}
\toprule
\textbf{Algorithm} & \textbf{Served} & \textbf{Util.} & \textbf{Fair.} & \textbf{Conv.} \\
\midrule
Adaptive (ours)      & 66--93\% & \textbf{100\%} & 0.99 & 1 \\
Weighted Max-Min$^\S$ & 66--95\% & 97\% & 0.99 & 1 \\
LP Optimal (1-class) & 66--92\% & 97\% & 0.99 & 1 \\
LP Optimal (resv.)   & 66--93\% & 97\% & 0.99 & 1 \\
LP Optimal (2-class) & 55--91\% & 97\% & 0.98 & 1 \\
Proportional         & 55--90\% & 97\% & 0.98 & 1--107$^\dagger$ \\
Static Uniform       & 59--69\% & 65\% & 0.95 & 1 \\
Divide-by-Breach     & 75--98\% & 7\%$^\ddagger$ & 0.99 & 1--21$^*$ \\
\bottomrule
\multicolumn{5}{l}{\footnotesize $^*$Fails to settle under distributed attacks (5 of 40 runs).} \\
\multicolumn{5}{l}{\footnotesize $^\dagger$Slow to settle under distributed attacks.} \\
\multicolumn{5}{l}{\footnotesize $^\ddagger$Low utilization reflects gross over-allocation, see text.} \\
\multicolumn{5}{p{0.92\columnwidth}}{\footnotesize $^\S$Higher high-priority served, but by holding the best-effort class below its reservation split, a harder priority policy rather than a strict improvement, see text.} \\
\multicolumn{5}{p{0.92\columnwidth}}{\footnotesize Served ranges span the 8 scenarios (min--max of per-scenario means); they reflect scenario spread, not seed variance. Each per-scenario mean is over 5 seeds with a 95\% CI within $\pm 0.5$ points.}
\end{tabular}
\end{table}

These numbers tell a more nuanced story than utilization alone would suggest, and we
read them honestly. On \emph{high-priority served}, the proposed algorithm
(66--93\%) is competitive with the single-class LP optimum (66--92\%) and the
reservation-respecting LP (66--93\%), and ahead of the two-class LP and proportional
baselines, but it does not top the table on this axis. Two allocators report higher
high-priority ranges, and neither is a free lunch. Divide-by-breach reports the highest
range (75--98\%), but only as an artifact of gross over-allocation: it hands almost every
region far more capacity than the budget permits, which serves high-priority demand well in
simulation but corresponds to exactly the downstream-overload condition the system exists to
prevent (its 7\% utilization quantifies the over-allocation), so we treat it as a cautionary
reference, not a target. Weighted max-min reports 66--95\% and does respect the budget, but
it buys the extra high-priority service by pulling capacity out of the best-effort class
below the balanced reservation split, a stronger priority policy rather than a strict
improvement; we quantify that tradeoff in Section~\ref{sec:served-results}. The served range
is wide because contention varies by scenario; the lower end (66\%) is the distributed
scenario where demand greatly exceeds capacity everywhere, and the upper end (93\%) is the
pulse-wave scenario. Fairness is reported to two decimals; the demand-weighted Jain's index
for the demand-aware allocators is 0.94--0.99, discussed in
Section~\ref{sec:fairness-results}. The utilization column is the subject of the next
subsection.

\subsubsection{Utilization}

Figure~\ref{fig:utilization} shows system-wide capacity utilization, the fraction of allocated capacity matched by demand across both classes. Our algorithm reaches 100\% across all scenarios, the LP allocators and proportional sit near 97\%, static uniform at 65\%, and divide-by-breach at 7\%. The adaptive algorithm reaches 100\% because both of its levels are demand-driven: the inter-class step refuses to leave one class's capacity idle while the other class has unmet demand, and the intra-class step does the same across regions. The remaining 3\% gap to the LP allocators arises because the LP baselines recompute a static per-cycle optimum and can leave small residuals when demand shifts within a cycle.

We are deliberately not making utilization the headline result, because high utilization is necessary but not sufficient: an allocator can be fully utilized while serving the wrong demand. Divide-by-breach sits at the opposite extreme, its 7\% utilization quantifies the gross over-allocation (handing most regions the full budget) that overwhelms the shared downstream resource, the very failure mode the system is built to prevent. Utilization is best read as a guardrail (our algorithm wastes no capacity and never over-commits) rather than as the measure of merit. This full-utilization guarantee is conditional: it holds when aggregate demand meets or exceeds the budget, so that there is enough demand to cover the conserved capacity. All 8 scenarios are oversubscribed by construction (total demand exceeds the budget), which is the regime the system is built for; when aggregate demand falls below the budget no budget-respecting allocator can reach 100\%, because the demand shortfall is idle by definition. The measure of merit is high-priority demand served, which we turn to next.

\begin{figure}[t]
    \centering
    \includegraphics[width=\columnwidth]{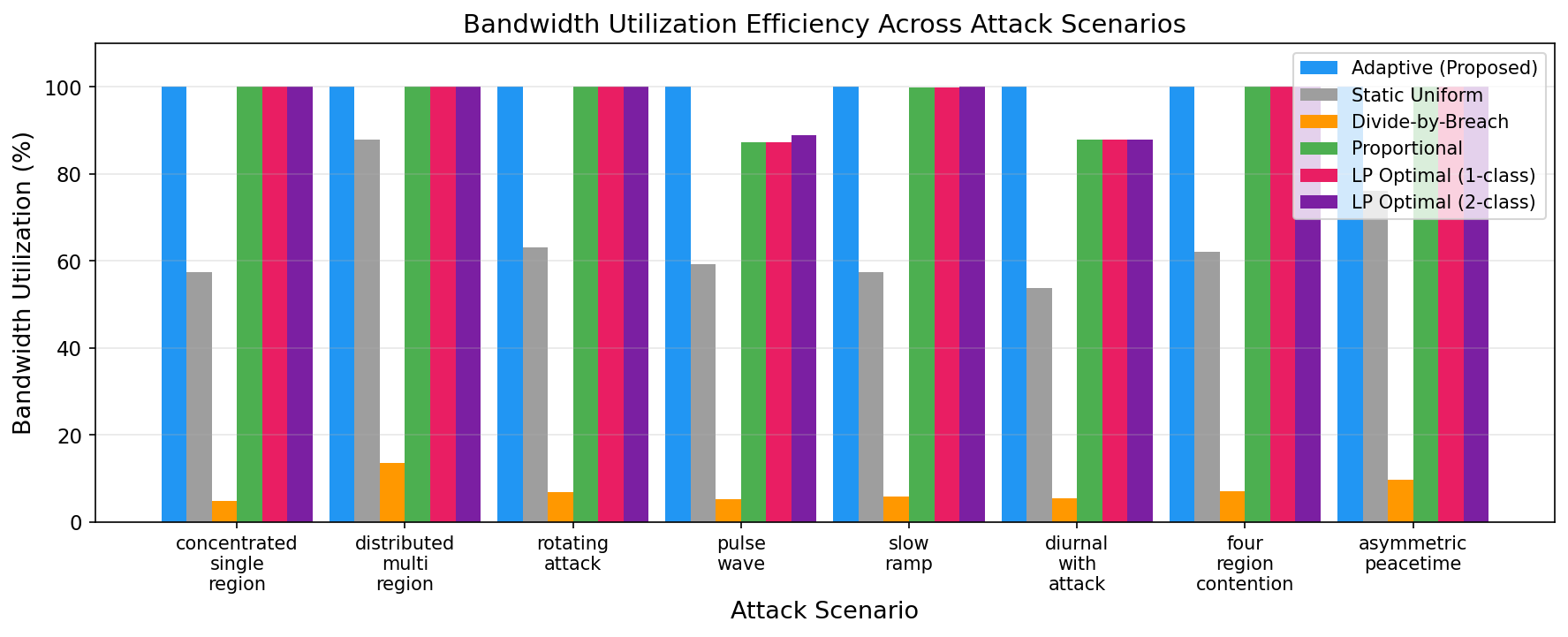}
    \caption{System-wide capacity utilization across 8 contention scenarios. The proposed algorithm wastes no capacity (100\%) without over-committing; divide-by-breach's 7\% reflects gross over-allocation that overwhelms the shared downstream resource.}
    \label{fig:utilization}
\end{figure}

\subsubsection{High-Priority Demand Served}
\label{sec:served-results}

High-priority demand served is the metric that matters, and it is measured on the class that carries the contending load: when contending load (in the CDN instantiation, application-layer attack traffic indistinguishable from legitimate requests) shares the high-priority class, genuine high-priority demand and contending load compete for the same allocation, and a region's high-priority service is its proportional share of what its allocation can cover. Figure~\ref{fig:served} shows the result across all scenarios.

The proposed algorithm serves 66--93\% of high-priority demand depending on scenario severity. It is competitive with the single-class LP optimum (66--92\%) and the reservation-respecting LP (66--93\%), and ahead of proportional (55--90\%). Static uniform trails everything (59--69\%) because it ignores demand: a region with a large share of high-priority demand receives the same $1/22$ slice as an empty one. The reservation-respecting LP is worth reading as the optimal reference for what the proposed algorithm does: it guarantees each class its within-reservation floor and then pools the surplus optimally. In six of eight scenarios that surplus is zero under contention (each class's demand alone exhausts its reservation), so it coincides with the single-class LP; where a class is under-subscribed it pulls ahead, most visibly on pulse-wave (93.2\% versus the single-class LP's 91.8\%). The proposed algorithm is within a point of it in three scenarios and beats it in three others (asymmetric by 2.5 points, distributed and diurnal by under a point), while trailing by up to 6.1 points in the worst case (a concentrated single-region attack, where the reservation-respecting LP's global surplus placement beats the proposed algorithm's proportional redistribution), all while using no solver and carrying no per-cycle state.

The most informative comparison is with the two-class LP (55--91\%), the throughput optimum for the full problem. The proposed algorithm serves \emph{more} high-priority demand than this optimum in most scenarios (five of eight), and the direction of the effect, not the exact count, is the point. This is not a contradiction; it exposes a mismatch of objectives. The two-class LP maximizes total served load across both classes, so when the best-effort class is saturated with contending load, the LP pours capacity into serving that load because doing so raises total throughput. Our allocator instead holds capacity in proportion to per-class demand and reservation, which under contention keeps more capacity on the high-priority class. The lesson is narrower than a blanket ranking: \emph{when the objective is protecting high-priority service under contention, maximizing total throughput can work against that goal}. An allocator that maximizes total served load can serve less high-priority demand than a demand-proportional one, because the optimizer cannot tell that some of the load it is serving is the contention itself. We show this directionally (in five of eight scenarios), not as a universal ordering. A demand-aware, reservation-respecting allocator avoids that trap by construction. The reservation-respecting LP makes this sharp: it is the throughput optimum \emph{subject to} the per-class floor, and it serves at least as much high-priority demand as the unconstrained two-class LP in all eight scenarios (by up to 10.3 points, on the distributed scenario), so the effect is a property of the objective and the reservation constraint, not of our particular heuristic. We are explicit that the effect is directional rather than uniform: in the remaining scenarios the throughput optimum serves more, and in the most adverse case (a concentrated single-region attack) it leads our allocator by 5.5 points, so the takeaway is the existence and cause of the objective mismatch, not a clean sweep. These per-scenario differences are well outside seed noise: the paired per-seed gap against the two-class LP ranges from $+12.0$ points (asymmetric) to $-5.5$ points (concentrated), each with a 95\% confidence interval narrower than $\pm 0.3$ points across the 5 seeds, so the wins and the losses are both real rather than sampling artifacts.

Weighted max-min fairness serves more high-priority demand than the proposed algorithm in seven of eight scenarios (by 4.2 points on average across those seven, up to 8.9 on the concentrated attack), but this is a policy difference, not a strict improvement, and reading only the high-priority column hides it. Weighted max-min caps how much any single cell can draw at $\lambda w_c$, which under a concentrated attack limits the capacity the attacked high-priority cells can pull and so protects the rest, but it funds that protection by driving the best-effort class below the balanced reservation split. On the concentrated attack it serves 89.4\% of high-priority demand against our 80.5\%, while its best-effort service falls to 30.5\% against our 40.8\%; the same pattern holds on rotating (best-effort 34.4\% versus 44.3\%) and four-region (35.7\% versus 44.2\%). The proposed algorithm and the reservation-respecting LP hold best-effort service at an identical level in every scenario, with the single-class LP matching them in seven of eight (it leaves a small residual only on pulse-wave), because all three respect the reservation split; weighted max-min sits below them and the throughput-maximizing two-class LP sits above them. Which point on that spectrum is correct is a deployment choice about how hard to prioritize the protected class at the expense of the other, not a question this evaluation settles. We report weighted max-min to make the choice explicit and to show that the proposed algorithm occupies the reservation-respecting middle, matching the LP references on best-effort service while using no solver.

We do not claim the top of the table. Divide-by-breach reports the highest high-priority range (75--98\%), but only by over-allocating so severely (7\% utilization) that it would overwhelm the shared downstream resource in deployment, which is the problem statement of this paper, not a solution to it. Weighted max-min reports the next highest (66--95\%) by trading away best-effort service as just described. Among the allocators that both respect the capacity budget and hold the reservation split, the proposed algorithm, the reservation-respecting LP, and the single-class LP lead on high-priority service, and the proposed algorithm alone does so with no solver, no per-cycle state, and full utilization.

\begin{figure}[t]
    \centering
    \includegraphics[width=\columnwidth]{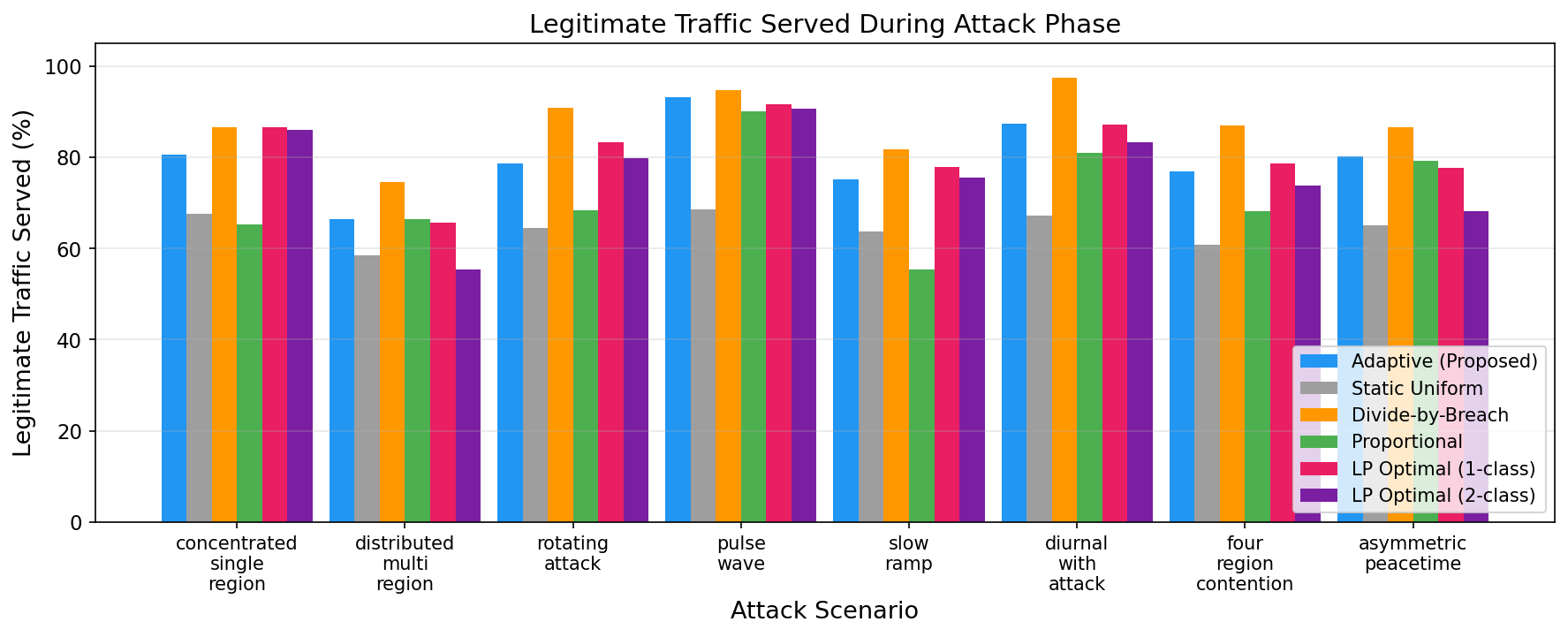}
    \caption{High-priority demand served during the contention phase across all eight scenarios. The proposed algorithm is competitive with the single-class LP optimum and serves more than the throughput-optimal two-class LP in most scenarios, because maximizing total throughput can mean serving the contending load at the expense of high-priority demand.}
    \label{fig:served}
\end{figure}

\subsubsection{Bursty Load: where inter-class borrowing earns its place}

The pulse-wave scenario is where the two-level design shows a clear, isolated benefit. Our algorithm serves 93.1\% of high-priority demand, against 91.8\% for the single-class LP and 90.7\% for the two-class LP. To confirm this advantage comes from inter-class borrowing specifically, and not from the intra-class redistribution that the baselines also approximate, we ran an ablation: the same algorithm with the inter-class step disabled (intra-class only), realized as a registered variant of the allocator and run through the identical experiment harness. The ablated version serves 91.6\%, the full algorithm 93.1\%, so the entire 1.5-point gain is attributable to inter-class borrowing. This gain is stable across seeds: over the 5 seeds the paired per-seed difference is $+1.49$ points (95\% CI $[1.47, 1.51]$, paired $t$-test), far outside seed noise. Across the other seven scenarios the same ablation changes high-priority service by 0.1 point or less, so the mechanism is neutral away from bursty load.

The gain is also not an artifact of the pulse scenario's chosen contention split. The split between the two classes within the excess load is a scenario parameter (the fraction of attack volume that lands in the high-priority class), and pulse fixes it at 0.2. Sweeping that fraction from 0.0 to 0.8 on the pulse scenario, the inter-class borrowing gain is a flat $+1.49$ points at every value, rising only at the degenerate end where nearly all excess is high-priority ($+1.86$ points at 0.9, $+4.05$ points at 1.0). The same sweep on the two non-bursty scenarios (distributed and four-region) yields essentially no gain across 0.0 to 0.8 (exactly 0.00 points through 0.6, and under 0.1 point at 0.8), and spikes only at the same degenerate split of 1.0, in fact higher there than pulse ($+8.6$ points on distributed and $+6.1$ points on four-region), because at that corner all excess is high-priority and nothing remains in the best-effort class to contend, so borrowing is pathological for every scenario alike. That all three scenarios spike at the split of 1.0 confirms the corner is an artifact of the pathological all-high-priority split, not evidence about burstiness. Across the realistic 0.0 to 0.8 range the benefit tracks temporal burstiness, not a hand-picked contention split: it holds across the full range of splits on the bursty scenario and is absent across the same range on the stationary ones.

The mechanism is temporal. During each off-peak window, best-effort demand drops, freeing capacity; inter-class borrowing reclaims it for the high-priority class for the duration of the lull, then returns it when the next burst arrives. A per-cycle optimizer, single-class or two-class, recomputes from the current demand each cycle and has no notion of holding or returning capacity across the cycle, so it cannot capture this gain. This is the honest core of the contribution: inter-class elastic borrowing helps high-priority service specifically under temporally bursty load, and is neutral, neither helping nor hurting, when demand is stationary.

\subsubsection{Convergence}

Figure~\ref{fig:convergence} shows convergence behavior. Our algorithm, both LPs, and static uniform settle in a single iteration in every scenario, because each recomputes allocations from the current demand and a fixed base reservation without feeding the previous cycle's output back as input. Computing from current demand is necessary but not sufficient for one-step settling: the proportional baseline also recomputes from scratch each cycle, yet on the distributed scenario it never settles (green bar, 59 iterations), because it passes per-region demand straight through with no reservation floor or damping. When the attack spreads large, noisy demand across all 22 regions at once, the per-region shares jitter with the $\pm30\%$ demand noise and the largest single-region reallocation stays above the 5\% convergence threshold from one cycle to the next. The proposed algorithm avoids this because it reallocates only the surplus or deficit relative to the fixed reservation base, which bounds the per-cycle movement and damps the demand noise rather than tracking it. This is the same failure mode as divide-by-breach reaching it by a different route (proportional tracks noise linearly; divide-by-breach jumps discontinuously as regions cross the over-limit threshold), and it is exactly the instability the one-step-stabilization result rules out for our allocator.

Divide-by-breach fails to stabilize under distributed attacks (shown as -1 in the data). The cause is not a feedback loop, the over-limit set is determined entirely by exogenous per-region demand against a fixed threshold, and does not depend on the previous allocation. Rather, the allocation is a discontinuous function of the over-limit count: each region that crosses the threshold changes the divisor, which shifts every region's budget at once. Under a distributed attack many regions sit near the threshold, so per-cycle demand noise repeatedly pushes them across it, the count fluctuates, and the allocation never settles. Under slow-ramp attacks the over-limit set instead grows monotonically as the ramp proceeds, adding regions one at a time, which is why stabilization tracks the ramp duration rather than oscillating. The proposed algorithm avoids both behaviors because its allocation is a continuous, demand-proportional function of the current demand vector.

\begin{figure}[t]
    \centering
    \includegraphics[width=\columnwidth]{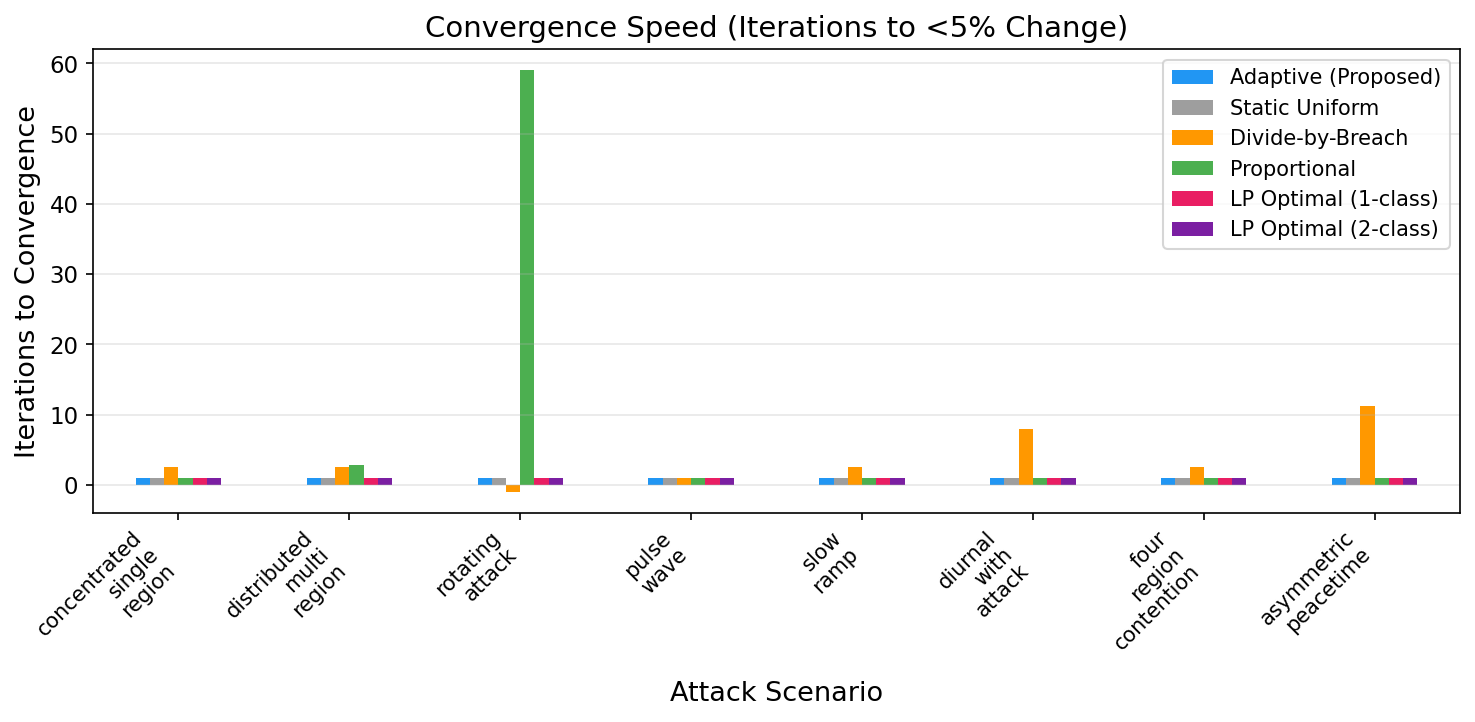}
    \caption{Convergence iterations. Two baselines fail to settle under the distributed scenario. Divide-by-breach (orange) never settles because demand noise repeatedly moves near-threshold regions in and out of the over-limit set, shifting the divisor each cycle. The proportional baseline (green) spikes to 59 iterations because it tracks the $\pm30\%$ per-region demand noise directly, with no reservation floor to damp it, so the largest per-region share keeps moving by more than the 5\% threshold. The proposed algorithm, both LPs, and static uniform settle in one iteration in every scenario.}
    \label{fig:convergence}
\end{figure}

\subsubsection{Fairness Under Scarcity}
\label{sec:fairness-results}

Figure~\ref{fig:fairness} shows demand-weighted fairness over time during a rotating attack, measured as Jain's index over per-region satisfaction ratios (high-priority demand served divided by high-priority demand). The takeaway of this subsection is narrow and we state it as such: among the budget-respecting allocators, fairness is \emph{not} a distinguishing axis. The demand-aware algorithms (ours, all three LPs, weighted max-min, and proportional) all hold the index in a narrow band, effectively tied, because each routes capacity toward demand and so equalizes satisfaction ratios across regions by construction. Our algorithm stays between 0.97 and 0.99 across every scenario; the band widens to 0.94--0.99 once the two-class LP is included, whose floor of 0.94 on the distributed scenario is the one demand-aware outlier. The two new baselines fall inside this band (weighted max-min 0.96--0.99, the reservation-respecting LP 0.96--0.99), so adding them does not change the conclusion that fairness is not a distinguishing axis among budget-respecting allocators. The one allocator that does separate is static uniform, which sits near 0.95 because it ignores demand entirely and persistently under-serves high-traffic regions relative to quiet ones. The index does not reach exactly 1 because demand is discrete and carries per-region noise, and because under contention some regions are more contended than others; the small oscillations in the figure are these per-cycle adjustments, not a drift away from fairness. We include fairness to show the proposed algorithm gives up nothing on this axis, not as a dimension on which it wins.

\begin{figure}[t]
    \centering
    \includegraphics[width=\columnwidth]{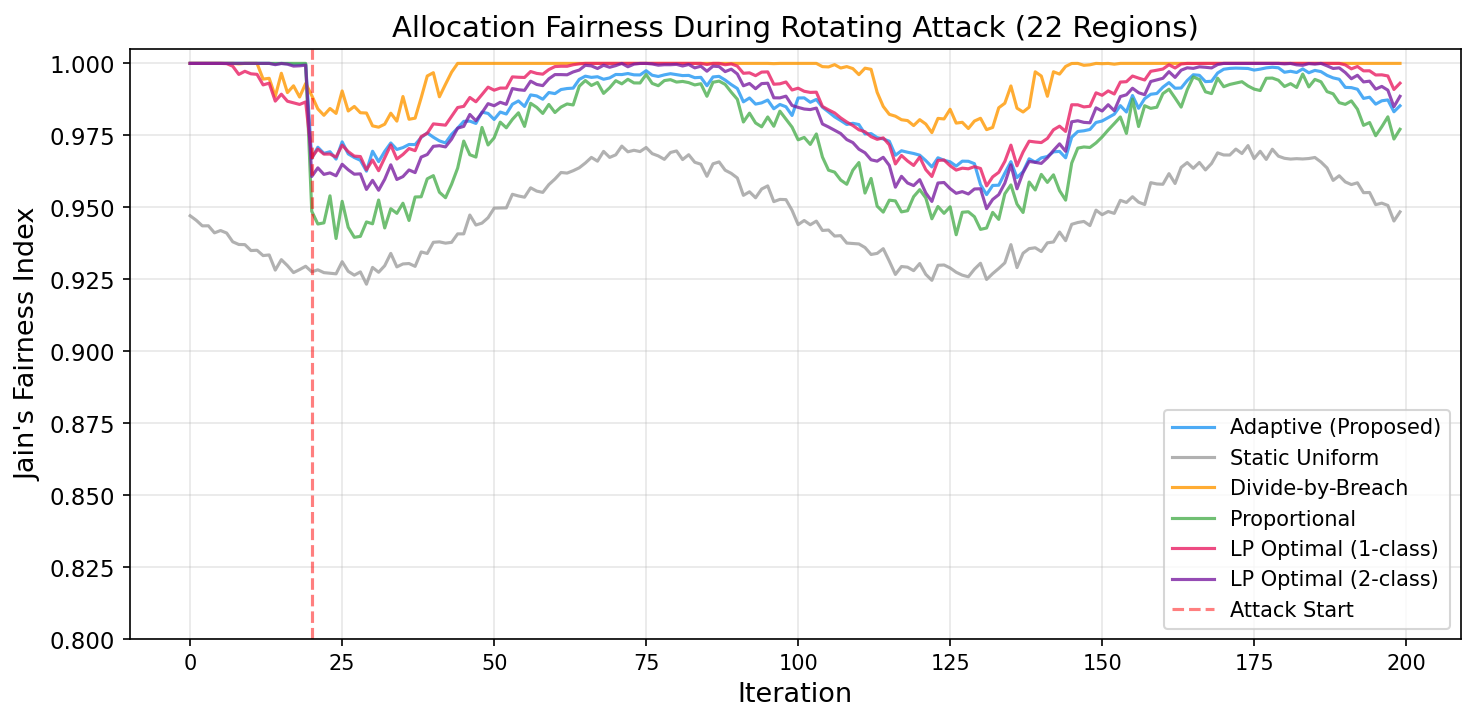}
    \caption{Demand-weighted fairness during rotating attack, y-axis scaled to [0.8, 1.0]. The demand-aware allocators are effectively tied between 0.97 and 0.99; only static uniform separates, penalizing high-traffic regions near 0.95.}
    \label{fig:fairness}
\end{figure}

\subsubsection{Utilization Over Time}

Figure~\ref{fig:util_time} shows system-wide utilization over time during a concentrated single-region attack. Our algorithm holds at 100\% throughout, while the other budget-respecting allocators dip slightly in low-demand phases of the diurnal cycle, where a class's demand falls below its reservation and the single-class methods leave the difference idle. This is the guardrail property discussed above, and it holds under contention (aggregate demand at or above the budget): the proposed algorithm never leaves capacity idle while unmet demand exists elsewhere, and never over-commits. We stress that this efficiency is not by itself the contribution, it is a precondition; the high-priority-service results above are what distinguish the allocators.

\begin{figure}[t]
    \centering
    \includegraphics[width=\columnwidth]{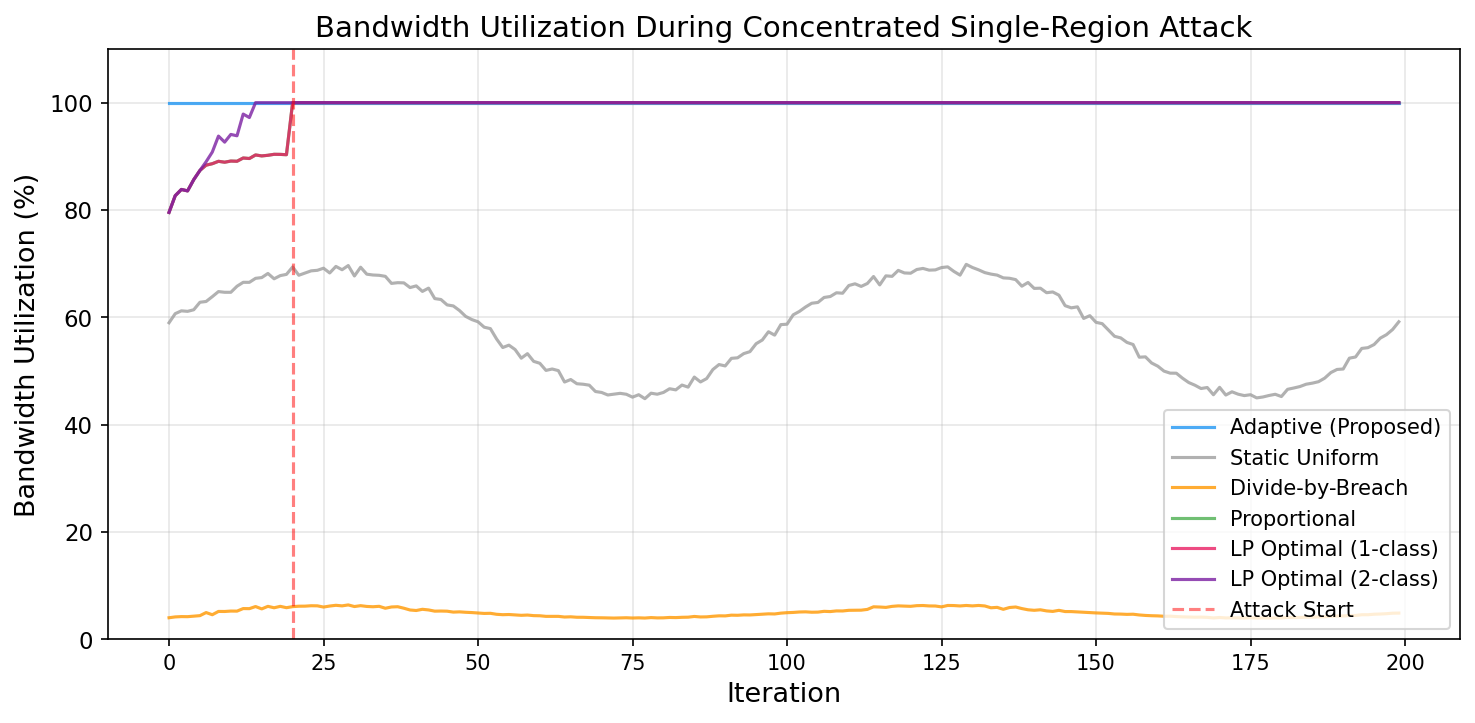}
    \caption{System-wide utilization over time during concentrated attack. The proposed algorithm sustains 100\% while the other budget-respecting allocators leave small residuals idle during low-demand phases.}
    \label{fig:util_time}
\end{figure}

\section{Counting-Pipeline Validation}
\label{sec:deployment}

We validated the end-to-end counting and control pipeline of the CDN instantiation in a 5-region prototype deployment running on Docker. This experiment exercises the data path that feeds the allocator, namely per-pod counting, hierarchical aggregation, and cross-region transport. It does not by itself validate the allocation decisions, which are evaluated in Section~\ref{sec:evaluation}; we are explicit about that boundary below. Each region contains a production-equivalent Nginx proxy with the Rust counting module, an Envoy enforcer with WAF integration, a Local Control Plane (LCP) for per-region aggregation, and a Redis instance. A Central Management Plane (CMP) connects to all 5 LCPs via bidirectional gRPC streams.

\subsection{Setup}

The deployment runs 42 containers across 5 isolated Docker networks (one per region) with uplink and client networks bridging them. The Rust pod counter module runs inside each Nginx proxy, counting per-domain requests using lock-free atomics and flushing deltas to the LCP every second.

Traffic is generated using curl through the proxy, with the full request path exercised: TLS termination, WAF evaluation, domain lookup from Redis, origin proxying, and response delivery. Rate limits are enforced based on allocations pushed from the control plane.

\subsection{Experiment}

We ran a 3-phase experiment:
\begin{enumerate}
    \item \textbf{Peacetime (30s):} 20 RPS legitimate traffic distributed evenly across all 5 regions.
    \item \textbf{Attack (60s):} 100 RPS concentrated on region-a, with legitimate traffic continuing on all regions.
    \item \textbf{Recovery (30s):} Attack stops, observe allocation return to baseline.
\end{enumerate}

\subsection{Results}

We logged per-region aggregated request counts (RPM) at the CMP. The data shows clear attack concentration and recovery:

During peacetime, all regions report approximately equal counts. When the attack begins, the targeted region's count rises well above the others while they remain at baseline. After the attack stops, the targeted region returns to baseline. The captured counts reflect the load the testbed sustained through the full proxy path rather than the nominal offered rate, and the logging cadence in this run averaged roughly one sample per minute rather than the 10-second control cadence; reconciling the prototype's logging units and cadence with the control plane is left to future hardening.

The LCPs tracked per-domain counters across all 5 regions, the CMP maintained gRPC connections to all clusters, and the counting pipeline (Rust FFI in Nginx, flush to LCP, aggregation at CMP) operated without errors throughout the experiment.

This demonstrates that the counting and aggregation pipeline, from pod-level atomic counting through hierarchical aggregation to the central plane, functions correctly with real HTTP traffic and production-grade proxy infrastructure. We do not claim it validates the allocation algorithm itself: the prototype logged input counts but not the allocations, rate limits, or served-traffic outcomes that would be needed to evaluate allocation decisions in deployment. Closing that gap, by logging allocator output and served traffic in the testbed, is the natural next step.

\section{Related Work}
\label{sec:related}

\textbf{Fair resource allocation.}
Weighted fair queueing~\cite{wfq} and its variants provide per-flow fairness at a single network element, and proportional share scheduling offers similar guarantees for CPU and memory. Dominant Resource Fairness~\cite{drf2011} generalizes max-min fairness to users with heterogeneous demands over multiple resource \emph{types}; our problem is different in shape, a single conserved resource shared across locations and two priority classes, but DRF's lesson that the fairness objective must match the structure of demand carries over, and we include a weighted max-min baseline in the DRF style in our evaluation (Section~\ref{sec:served-results}), where it serves more high-priority load than the proposed algorithm precisely by weighting the protected class more aggressively at the expense of the best-effort class. Our contribution is applying fairness principles to a two-dimensional problem (locations $\times$ two service classes) with the additional constraint that total capacity must be strictly conserved, and identifying that a throughput-maximizing objective is the wrong one when one class carries contending load. This positioning is stated for the general allocation problem; the DDoS setting is the instantiation in which we evaluate it.

The water-filling algorithm from information theory provides the theoretical optimum for single-dimensional allocation, and our three LP baselines implement it for the single-class, reservation-respecting, and pooled two-class cases respectively. The two-class LP is the throughput optimum for the full problem, so our algorithm cannot beat it on total served load within a cycle. The interesting gap is on \emph{high-priority} demand, where our demand-proportional allocator can exceed the throughput optimum precisely because the latter spends capacity on contending load to maximize total throughput; the reservation-respecting LP confirms this is a property of the objective under a per-class floor, not of our heuristic.

\textbf{Centralized capacity allocation in production systems.}
The closest systems to our setting are the centralized bandwidth allocators of production wide-area networks. SWAN~\cite{swan2013} centrally decides how much each service may send, allocating higher-priority classes first with weighted max-min fairness within each class; BwE~\cite{bwe2015} allocates WAN bandwidth hierarchically across services with mixed guaranteed and best-effort classes. Both solve a richer problem than ours (multiple paths, many services, bandwidth functions) with correspondingly heavier machinery, and neither faces our defining constraint that part of the observed demand may be adversarial load the allocator should not maximize. Taiji~\cite{taiji2019} manages global user traffic from edge nodes to data centers to balance utilization; it routes demand to capacity, whereas we distribute capacity to demand under a conserved budget. Our algorithm occupies a deliberately simpler point: two classes, one budget, solver-free $O(KN)$ proportional redistribution with provable conservation.

\textbf{DDoS defense at the CDN edge (the evaluated instantiation).}
Surveys of DDoS defense in cloud environments~\cite{ddos_cloud_survey} catalog the detection and mitigation landscape; the question we study, how to distribute a fixed capacity budget across locations while an attack is being absorbed, sits downstream of that landscape and is comparatively unaddressed. Alcoz et al.~\cite{alcoz2022} proposed aggregate-based congestion control for pulse-wave DDoS attacks using P4 programmable switches. Their system infers attack patterns via online clustering and applies per-packet rate limiting on Intel Tofino hardware. This operates at L3/L4 per-flow granularity, which is complementary to our per-domain geographic allocation. We address the higher-level question of how to distribute a domain's capacity budget across regions.

Recent work on coordinated cloud-edge DDoS scrubbing~\cite{tmc2024} addresses predictive resource coordination between scrubbing centers and edge locations. Their focus is on deciding when to activate scrubbing (a binary decision), while we address how to continuously redistribute a fixed capacity budget across many locations during an ongoing attack.

\textbf{Detection across distributed domains.}
Chen et al.~\cite{ieee2008} address collaborative detection of DDoS attacks across multiple network domains, correlating change-point signals between domains to identify attacks earlier than any single vantage point; more recent detection work spans programmable data planes~\cite{signature2021}, graph neural networks~\cite{deepgraph2024}, and distributed SDN-based prediction at the edge~\cite{edge_cloud2025}. All of this targets the detection problem, identifying that an attack is underway, which is upstream of and complementary to the allocation problem we study: we take detection and class labels as inputs and decide how to distribute a fixed capacity budget across locations once contention is recognized.

\textbf{Allocation and defense under attack.}
Kumar and Bhuyan~\cite{kumar2022} applied game theory to defend elastic and inelastic services against DDoS, modeling the interaction between attacker and defender as a two-player game. Their work provides theoretical bounds but does not address the practical problem of geographic distribution or per-class capacity sharing.

Cooperative DDoS defense across edge and cloud environments has been studied in~\cite{edge_cloud2025}, but this work focuses on detection coordination rather than capacity allocation.

\textbf{CDN resource management.}
Content placement and request routing in CDNs is well-studied, from the classical survey of Pallis and Vakali~\cite{cdn_survey} to production-scale edge traffic management~\cite{taiji2019}, but these systems optimize for latency, cache hit rates, and utilization balance, not for defending a fixed capacity budget under contention. Our evaluation is specific to the attack scenario where a domain's capacity budget must be defended across regions while maintaining service for legitimate users; the allocation model itself is not.

\textbf{Traffic classification for DDoS.}
Signature-based classification~\cite{signature2021} and deep learning approaches~\cite{deepgraph2024} focus on the detection problem: identifying which traffic is malicious. We take classification as input and address the resource allocation problem that follows: given labeled demand, how to distribute limited capacity fairly.

\section{Discussion and Limitations}
\label{sec:limitations}

\textbf{Generality beyond the evaluated instantiation.}
We evaluated the algorithm on CDN capacity defense because that is the problem that motivated it and the one for which we have a deployment and a calibrated traffic model. Nothing in the algorithm, its invariants, or its convergence result depends on the contending load being adversarial. The model is a conserved budget shared across locations and two priority classes, with demand that is skewed across locations and time-varying; any setting matching that shape is a candidate, for example splitting a fixed egress or database-connection budget across data centers between latency-critical and batch workloads, or sharing a licensed throughput cap across tenants between a premium and a standard tier. Two properties travel directly: conservation holds by construction for any non-negative demand (Theorems 1--3), and the inter-class borrowing benefit appears wherever one class's demand is temporally bursty so that it frees capacity the other class can use during lulls. Two properties are instantiation-specific and would need re-checking elsewhere: the concrete concentration of demand across locations (a parameter of our traffic model) and the existence of a meaningful priority split between the two classes (if both classes are equally critical, the inter-class step still conserves capacity but the ``protect the high-priority class'' framing no longer applies). We do not claim to have validated these other instantiations; we claim the model is not specific to DDoS and identify what a port would have to verify.

\textbf{Information staleness.}
The algorithm operates on demand observations that are 10-20 seconds old in our deployment, due to the hierarchical aggregation pipeline (Section~\ref{sec:system}). Under demand patterns that shift faster than the observation delay, allocations will lag behind reality. In our deployment, the attacks we have observed evolve on the timescale of minutes rather than seconds, so the 10-second decision cycle has been sufficient there; an instantiation with faster demand shifts would need a proportionally faster pipeline. The CRDT-based synchronization ensures that even under network partitions between locations, each location continues operating with the best available (possibly stale) information.

\textbf{Oscillation risk.}
Because the algorithm is reactive (it allocates based on observed demand), there is a theoretical risk of oscillation: high demand at a location triggers allocation, which satisfies the demand, which reduces observed demand, which reduces allocation. We argue analytically that this does not occur in our setting: high-priority demand is persistent (it does not vanish once served) and excess best-effort demand exceeds capacity (so allocation never fully satisfies it), so neither class's observed demand collapses in response to being served. Our property-based tests exercise the single-cycle invariants that underpin this argument (exact conservation and non-negativity across randomized demand vectors); they do not model demand dynamics across cycles, so the absence of oscillation rests on the analytical argument above rather than on a multi-cycle experiment. A time-series stability experiment over evolving demand is future work.

\textbf{Class-assignment dependency.}
The inter-class borrowing mechanism depends on demand being assigned to the right class. If a large share of high-priority demand is mislabeled best-effort, the high-priority class will be undersized. In the CDN instantiation the system degrades gracefully because the best-effort class is rate-limited rather than dropped, so mislabeled requests still get partial service, and as the classifier gains confidence they migrate back to the high-priority class. In a non-adversarial instantiation the analogous risk is a misconfigured priority tag; the same graceful-degradation argument applies as long as the best-effort class is served rather than denied.

\textbf{Single-budget scope.}
The algorithm operates independently per conserved budget (per domain in the CDN instantiation). If several budgets share a deeper downstream resource, per-budget allocation does not capture the cross-budget contention. Extending to coupled budgets is future work; in our deployment domains map to separate origin pools, so the budgets are genuinely independent.

\textbf{Cold start.}
When a budget first comes under contention with no prior allocation history, the algorithm starts from a uniform distribution and adapts within 1-2 cycles. During this cold-start window (up to 20 seconds in the CDN instantiation), some high-priority demand may be denied at heavy-demand locations. This is acceptable given the alternative of no adaptation at all.

\section{Conclusion}
\label{sec:conclusion}

We presented a two-level algorithm for allocating a conserved capacity budget across many locations and two service classes. The algorithm combines intra-class redistribution across locations with inter-class elastic borrowing between classes, maintains conservation and fairness invariants, runs in $O(KN)$ time, and reaches a stable allocation in a single iteration under stationary demand because it carries no per-cycle state. We evaluated it on the instantiation that motivated it, defending a CDN's per-domain budget under volumetric attack, across 8 contention scenarios on a 22-location topology, with contending load measured in the class it targets; there it serves high-priority demand competitively with a single-class LP optimum while never over-committing the budget.

Two findings are worth carrying forward, and both are about the problem rather than the application. First, a throughput-maximizing objective is the wrong objective under contention: a two-class LP that maximizes total served load serves less high-priority demand than our demand-proportional allocator in most scenarios, because it cannot distinguish high-priority load from the contending load it is also serving. A demand-aware allocator that respects per-class reservations avoids this by construction. Second, the added complexity of inter-class borrowing earns its place specifically under temporally bursty load: an ablation, realized as a registered allocator variant and run through the same harness, shows it improves high-priority service under pulse-wave load and is neutral (0.1 point or less) under stationary demand. We would rather state plainly where the mechanism helps and where it does not than claim a uniform advantage the data does not support. Validating the algorithm in further instantiations beyond CDN defense, capturing classifier dynamics in the evaluated one, and validating allocation decisions (not only the counting pipeline) in deployment, are the natural next steps.


\begin{thebibliography}{00}
\bibitem{alcoz2022} A.~Gran~Alcoz, M.~Strohmeier, V.~Lenders, and L.~Vanbever, ``Aggregate-Based Congestion Control for Pulse-Wave DDoS Defense,'' in \emph{Proc. ACM SIGCOMM}, 2022.
\bibitem{tmc2024} R.~Zhou, Y.~Zeng, L.~Jiao, Y.~Zhong, and L.~Song, ``Online and Predictive Coordinated Cloud-Edge Scrubbing for DDoS Mitigation,'' \emph{IEEE Trans. Mobile Computing}, vol. 23, no. 10, pp. 9208--9223, 2024.
\bibitem{ieee2008} Y.~Chen, K.~Hwang, and W.-S.~Ku, ``Collaborative Detection of DDoS Attacks over Multiple Network Domains,'' \emph{IEEE Trans. Parallel and Distributed Systems}, vol. 18, no. 12, pp. 1649--1662, 2007.
\bibitem{kumar2022} B.~Kumar and B.~Bhuyan, ``Using Game Theory to Defend Elastic and Inelastic Services Against DDoS Attacks,'' in \emph{Lecture Notes in Networks and Systems}, Springer, 2022.
\bibitem{edge_cloud2025} H.~Zhou, Y.~Zheng, X.~Jia, and J.~Shu, ``Collaborative Prediction and Detection of DDoS Attacks in Edge Computing: A Deep Learning-Based Approach with Distributed SDN,'' \emph{Computer Networks}, vol. 225, art. 109642, 2023.
\bibitem{wfq} A.~Demers, S.~Keshav, and S.~Shenker, ``Analysis and Simulation of a Fair Queueing Algorithm,'' in \emph{Proc. ACM SIGCOMM}, 1989.
\bibitem{drf2011} A.~Ghodsi, M.~Zaharia, B.~Hindman, A.~Konwinski, S.~Shenker, and I.~Stoica, ``Dominant Resource Fairness: Fair Allocation of Multiple Resource Types,'' in \emph{Proc. USENIX NSDI}, 2011, pp. 24--37.
\bibitem{swan2013} C.-Y.~Hong, S.~Kandula, R.~Mahajan, M.~Zhang, V.~Gill, M.~Nanduri, and R.~Wattenhofer, ``Achieving High Utilization with Software-Driven WAN,'' in \emph{Proc. ACM SIGCOMM}, 2013, pp. 15--26.
\bibitem{bwe2015} A.~Kumar \emph{et al.}, ``BwE: Flexible, Hierarchical Bandwidth Allocation for WAN Distributed Computing,'' in \emph{Proc. ACM SIGCOMM}, 2015.
\bibitem{taiji2019} D.~Chou \emph{et al.}, ``Taiji: Managing Global User Traffic for Large-Scale Internet Services at the Edge,'' in \emph{Proc. ACM SOSP}, 2019, pp. 430--446.
\bibitem{ddos_cloud_survey} N.~Agrawal and S.~Tapaswi, ``Defense Mechanisms Against DDoS Attacks in a Cloud Computing Environment: State-of-the-Art and Research Challenges,'' \emph{IEEE Communications Surveys \& Tutorials}, vol. 21, no. 4, pp. 3769--3795, 2019.
\bibitem{cdn_survey} G.~Pallis and A.~Vakali, ``Insight and Perspectives for Content Delivery Networks,'' \emph{Communications of the ACM}, vol. 49, no. 1, 2006.
\bibitem{signature2021} M.~Dimolianis, A.~Pavlidis, and V.~Maglaris, ``Signature-Based Traffic Classification and Mitigation for DDoS Attacks Using Programmable Network Data Planes,'' \emph{IEEE Access}, vol. 9, pp. 113061--113076, 2021.
\bibitem{deepgraph2024} L.~Barsellotti, L.~De~Marinis, F.~Cugini, and F.~Paolucci, ``FTG-Net: Hierarchical Flow-to-Traffic Graph Neural Network for DDoS Attack Detection,'' in \emph{Proc. IEEE Int. Conf. High Performance Switching and Routing (HPSR)}, 2023, pp. 173--178.
\end{thebibliography}
\end{document}